\documentclass{article}
\usepackage{ijcai26}

\usepackage{times}
\usepackage{soul}
\usepackage{url}
\usepackage[hidelinks]{hyperref}
\usepackage[utf8]{inputenc}
\usepackage[small]{caption}
\usepackage{graphicx}
\usepackage{amsmath}
\usepackage{amsthm}
\usepackage{booktabs}
\usepackage{algorithm}
\usepackage{algorithmic}
\usepackage[switch]{lineno}

\usepackage{graphicx} 
\usepackage{amssymb}

\usepackage{subcaption}
\usepackage{tabularx}
\usepackage{multirow}

\usepackage{xcolor}
\usepackage{ulem}
\usepackage{array}

\newtheorem{theorem}{Theorem}

\title{E-TCAV: Formalizing Penultimate Proxies for Efficient Concept Based Interpretability}

\author{
Hasib Aslam$^{1,*\dagger}$,
Muhammad Ali Chattha$^{2,*}$,
Muhammad Taha Mukhtar$^{1,2}$,
Muhammad Imran Malik$^1$,
Andreas Dengel$^2$,
Sheraz Ahmed$^2$
\\
\affiliations
$^1$National University of Sciences and Technology (NUST)\\
$^2$German Research Centre for Artificial Intelligence (DFKI)\\
\emails
hasibaslam4152@gmail.com,
muhammad\_ali.chattha@dfki.de,
mmukhtar.bscs21seecs@seecs.edu.pk,
malik.imran@seecs.edu.pk,
andreas.dengel@dfki.de,
sheraz.ahmed@dfki.de
\\
$^*$Equal contribution.
$^\dagger$Corresponding author.
}

\begin{document}

\maketitle
\section{Abstract}
TCAV (Testing with Concept Activation Vectors) is an interpretability method that assesses the alignment between the internal representations of a trained neural network and human-understandable, high-level concepts.
Though effective, TCAV suffers from significant computational overhead, inter-layer disagreement of TCAV scores, and statistical instability.
This work takes a step toward addressing these challenges by introducing E-TCAV, a framework for efficient approximation of TCAV scores, which is based on extensive investigation into three key aspects of the TCAV methodology: 1) the effect of latent classifiers on the stability of TCAV scores, 2) the inter-layer agreement of TCAV scores, and 3) the use of the penultimate layer as a fast proxy for earlier layers for TCAV computation.
To ensure a solid foundation for E-TCAV, we conduct extensive evaluations across four different architectures and five datasets, encompassing problems from both computer vision and natural language domains.
Our results show that the layers in the final block of the neural network strongly agree with the penultimate layer in terms of the TCAV scores, and the commonly observed variance of the TCAV scores can be attributed to the choice of the latent classifier.
Leveraging this inter-layer agreement and the degeneracy of directional sensitivities at the penultimate layer, E-TCAV guarantees linearly scaling speed-ups with respect to the 
network's size and the number of evaluation samples, marking a step towards efficient model debugging and real-time concept-guided training.\footnote{Code is available at: \url{https://github.com/hasib2003/E_TCAV}}

\begin{figure}[t!]
    \centering
    \includegraphics[width=1\linewidth]{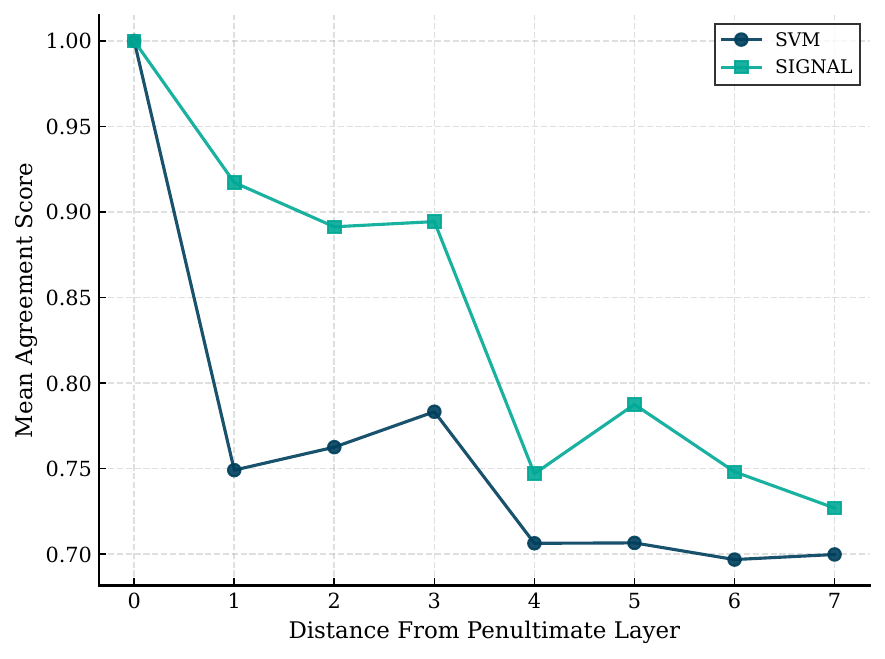}
    \caption{Behavior of TCAV scores at different layers. The x-axis is the distance of the layer from the penultimate layer; the y-axis represents the average agreement of TCAV scores between the target and penultimate layer. The results are averaged across five distinct datasets and four different neural network architectures.}
    \label{fig:alignment_summary}
\end{figure}
\section{Introduction}
Testing with concept activation vectors (TCAV)~\cite{kim2018interpretability}, is a widely used technique for post-hoc interpretability of deep neural networks. The main idea of TCAV is to quantify the effect of human-understandable concepts, commonly represented as vectors, on the model's prediction by measuring how prediction confidence changes as latent representations of inputs align more closely with those concept vectors.
Despite the TCAV's ability to bridge the black box neural networks to high-level human-understandable concepts, multiple challenges complicate its applicability, including its computational complexity and variability across layers. 

The computation of the TCAV score as proposed in~\cite{kim2018interpretability} is a two-step procedure.
In the first step, a concept vector is calculated by training a linear binary classifier to separate the latent representation of concept samples (typically images or text) from those of non-concept samples at a specific layer of the model. 
In the second step, the TCAV score is then computed as the fraction of input samples for which the model’s prediction confidence for a given class increases when the latent representation is perturbed in the direction of the concept vector, capturing the model’s directional sensitivity to the concept.
This two-step TCAV procedure incurs significant computational complexity, especially when TCAV is integrated into the training process rather than used as a post-hoc interpretability tool, as in TCAV-based guided training approaches~\cite{dreyer2024hope}.
It is important to note that while FastCAV~\cite{schmalwasser2025fastcav} improves the efficiency of computing the concept activation vector, the second step of TCAV, calculating directional sensitivities, remains computationally expensive.

Another significant challenge of TCAV is the variability of the scores both across different CAV runs and across network layers. This is reflected in the statistical insignificance of TCAV scores and the lack of agreement across different layers regarding the presence of a concept. As a result, a concept may appear to be present at certain layers while being absent at others, making it difficult for experts to draw clear and decisive conclusions. 

In this study, we work towards addressing these limitations.
We first demonstrate that the commonly observed variability of TCAV scores, as reported in~\cite{nicolson2024explaining}, can be attributed to the choice of latent classifiers and show a substantial reduction in score variance, compared to the commonly used SVM, by adopting SignalCAV \cite{pahde2022navigating} as the latent classifier.
Extensive experiments conducted across five benchmark datasets and four neural network architectures show that this choice consistently yields low-variance TCAV scores and significantly enhances their statistical significance.

We further show that using SignalCAV as a latent classifier results in more consistent TCAV scores among layers in the final block of the neural network.
The semantic representations become increasingly similar towards the end of the neural network, resulting in strong inter-layer agreement between the TCAV scores computed at the penultimate layer and those from earlier layers, which decrease gradually as we go deeper in the network.
A key distinction from a recent work by~\cite{nicolson2024explaining} is that we focus on TCAV agreement between the penultimate layer and preceding layers, rather than on general inter-layer agreement between arbitrary layer pairs.

Analytically, we prove that under commonly occurring conditions, the directional sensitivity at the penultimate layer does not depend on the input samples, and the iterative computations are reduced to a single dot product.
Leveraging this analytical observation and inter-layer TCAV agreement, we propose E-TCAV, an approximation framework for efficient calculation of TCAV scores. To the best of our knowledge, we are the first ones to formally present these observations to the community.

The main contributions of this research can be summarized as follows:
\begin{itemize}

    \item Through extensive experiments on four architectures and five benchmark datasets, we show that SignalCAV induces strong correlations between TCAV scores from the penultimate layer and those from earlier layers.
    \item We introduce E-TCAV, an approximation framework for TCAV scores that guarantees linearly scaling speed-ups with respect to network size and number of evaluation samples.  
    \item We show that the variance of the TCAV scores can be attributed to the choice of the latent classifier, and the use of Signal CAV reduces the standard deviation and enhances statistical significance.
        
\end{itemize}

\section{Preliminaries}
\label{sec:preliminaries}

Consider an arbitrary deep neural network-based classifier, without the loss of generality, we can define it as $h \circ f$ where $f$ is interpreted as features extractor and $h$ as a classifier, giving the logits for different classes. It is important to note that $h$ and $f$ represent an arbitrary division of the neural network and may not necessarily correspond to the traditional feature extractor and classifiers. In general 
\begin{align}
f &: \mathbb{R}^{d_1 \times d_2} \to \mathbb{R}^m \\
h &: \mathbb{R}^m \to \mathbb{R}^c
\end{align}
where $d_1, d_2$ are arbitrary input dimensions, $m$ is the dimension of the latent space and $c$ is the number of classes. We use $h_k$ to denote the function that outputs the logit for the $k^{th}$ class only.

Let $\boldsymbol{v}_C \in \mathbb{R}^m$ be the concept vector for some particular concept $C$, following~\cite{kim2018interpretability}, the "true conceptual sensitivity" of class $k$ to the concept $C$ can be computed as the directional derivative $S_{C,k}$ defined as: 
\begin{equation}
S_{C,k}(\boldsymbol{x})
\;=\;
\nabla_{\boldsymbol{x}} h_k(\boldsymbol{x}) \cdot\ \boldsymbol{v}_C
\label{def:directional-derivative}
\end{equation}

The TCAV score is then computed as the fraction of samples in class
$\boldsymbol{X}_k$ whose directional derivative along the concept
vector $\boldsymbol{v}_C$ is positive:
\begin{equation}
\mathrm{TCAV}_{C,k} = 
\frac{ \bigl| \{ \boldsymbol{x} \in \boldsymbol{X}_k \;|\; S_{C,k}(\boldsymbol{x}) > 0 \} \bigr| }
     { |\boldsymbol{X}_k| }
\label{def:TCAV}
\end{equation}
For the rest of the sections, we follow the notation of making vectors lower-case and bold, matrices are upper-case and bold, while scalars are lower-case and unbold.

When performing the TCAV testing, for each concept we extracted 30 different concept activations (unless stated otherwise) and performed a two-sided t-test with alpha of $0.05$ to ensure the statistical significance. The statistically insignificant results are labeled with a \textcolor{red}{*} symbol. 

\subsection{SignalCAV - Latent Classifier}
Testing with concept activation vectors, as introduced in~\cite{kim2018interpretability}, uses support vector machines (SVM) with linear kernels to find the decision boundary separating the latent representation of concepts and non-concept samples. \footnote{We refer to these classifiers as \textit{latent classifiers} to distinguish from the classifier $g$ defined in section \autoref{sec:preliminaries}}
Although multiple latent classifier such as Lasso, Ridge, FastCAV~\cite{schmalwasser2025fastcav}, or SignalCAV~\cite{pahde2022navigating}, can be used, we use the SignalCAV as it is invariant to the feature scaling and is more robust to noise~\cite{pahde2022navigating}.

Given a dataset $X = \{(\mathbf{h}, t)\}$ consisting of latent activations $\mathbf{h}$ and their corresponding binary labels $t$, the Signal CAV, denoted by $\mathbf{v}_{\mathrm{sig}}$, is defined as
\begin{equation}
\mathbf{v}_{\mathrm{sig}} =
\frac{1}{\sigma_t^2 \lvert X \rvert}
\sum_{(\mathbf{h}, t) \in X}
(\mathbf{h} - \bar{\mathbf{h}})(t - \bar{t}),
\end{equation}
where $\bar{\mathbf{h}}$ and $\bar{t}$ denote the empirical means of $\mathbf{h}$ and $t$, respectively, and $\sigma_t^2$ is the variance of $t$.

\section{The Affine Degeneracy}
\label{sec:theorem-fast-proxy}
In this section, we demonstrate that the conceptual sensitivity and the resulting TCAV score collapse to a sample-independent constant when the classifier $h$ is an affine mapping. This observation is critical to the setup of E-TCAV.

\begin{theorem}[Degeneracy of TCAV for Affine Classifiers]
\label{theorem:fast-penul}
Let $h_k : \mathbb{R}^m \to \mathbb{R}$ be an affine classifier for class $k$, defined as
$
h_k(\boldsymbol{a}) = \boldsymbol{w}_k^\top \boldsymbol{a} + b_k,
$
where $\boldsymbol{a} \in \mathbb{R}^m$ represents the activations from a bottleneck layer,
$\boldsymbol{w}_k \in \mathbb{R}^m$ is the weight vector, and $b_k \in \mathbb{R}$ is the bias.
For any concept vector $\boldsymbol{v}_C \in \mathbb{R}^m$, the TCAV score
$\mathrm{TCAV}_{C,k}$ is invariant to the input distribution $X_k$ and is determined
solely by the inner product $\langle \boldsymbol{w}_k, \boldsymbol{v}_C \rangle$.
\end{theorem}

\begin{proof}
The conceptual sensitivity $S_{C,k}(\boldsymbol{a})$ is defined as the directional derivative of the logit $h_k$ along the concept vector $\boldsymbol{v}_C$:
\[
S_{C,k}(\boldsymbol{a}) = \nabla_{\boldsymbol{a}} h_k(\boldsymbol{a}) \cdot \boldsymbol{v}_C.
\]

Substituting the affine form of $h_k$:
\[
S_{C,k}(\boldsymbol{a}) = \nabla_{\boldsymbol{a}} (\boldsymbol{w}_k^\top \boldsymbol{a} + b_k) \cdot \boldsymbol{v}_C
= \boldsymbol{w}_k \cdot \boldsymbol{v}_C.
\]

The sensitivity $S_{C,k}(\boldsymbol{a})$ is constant for all $\boldsymbol{a} \in \mathbb{R}^m$. By the definition of the TCAV score as the fraction of inputs with positive sensitivity:
\[
\mathrm{TCAV}_{C,k} = 
\frac{ \bigl| \{ \boldsymbol{x} \in X_k : S_{C,k}(g(\boldsymbol{x})) > 0 \} \bigr| }{|X_k|}.
\]

Since $S_{C,k}$ is constant and independent of $\boldsymbol{x}$, the condition $S_{C,k} > 0$ is either true for all $\boldsymbol{x}$ or false for all $\boldsymbol{x}$. Thus:
\[
\mathrm{TCAV}_{C,k} = \mathbb{I}(\boldsymbol{w}_k \cdot \boldsymbol{v}_C > 0),
\]
where $\mathbb{I}(\cdot)$ is the indicator function.
\end{proof}
Though straightforward, the implications of~\autoref{theorem:fast-penul} are significant.
One, it ensures that the TCAV scores can be calculated much faster at a layer, if the subsequent layer(s) constitute an affine map, since it does require computation of directional derivatives with respect to each evaluation sample. 
This is usually the case for the average pool layer in the convolutional neural networks, such as in ResNet, Inception, and DenseNet architectures, where average pooling is followed by fully connected layers to obtain class logits. 
\footnote{Some networks apply dropout after average pooling, but it is done only while training, and during inference, it is just an identity operation.}

Secondly, if there exists strong agreement between penultimate layers, for which~\autoref{theorem:fast-penul} holds, and previous layers, we can directly approximate the scores at previous layers using the penultimate one. In the next sections, we define a new metric for quantifying the TCAV agreement between two layers, followed by a comprehensive evaluation on diverse 
datasets to check if this agreement exists.
\section{The TCAV Agreement Score}
\label{metric:agreement}
To capture the total logical consistency between layers, we define the TCAV agreement score $\bar{A}(l_1,l_2)$, which evaluates the degree to which two layers, $l$ and $l_p$, agree on both the presence and absence of conceptual significance across the entire library $\mathcal{C}$.

Let $\mathbb{I}(c, l,\alpha)$ be an indicator function such that $\mathbb{I}(c, l) = 1$ if the TCAV score $\mathrm{TCAV}_{c,l} > \alpha$, and $0$ otherwise. The alignment is defined as the normalized sum of agreements across all concepts:

\begin{align}
A_{l, l_p, \alpha}
&=
\frac{1}{|\mathcal{C}|}
\sum_{c \in \mathcal{C}}
\Bigl[
    \mathbb{I}(c, l, \alpha)\,
    \mathbb{I}(c, l_p, \alpha)
\nonumber \\
&\qquad\quad
    +
    \bigl(1 - \mathbb{I}(c, l, \alpha)\bigr)
    \bigl(1 - \mathbb{I}(c, l_p, \alpha)\bigr)
\Bigr]
\label{eq:conceptual_state_alignment}
\end{align}

As it can be observed that the TCAV agreement score as defined in~\autoref{eq:conceptual_state_alignment} is sensitive to the value of $\alpha$, therefore instead of using arbitrary thresholds, we compute ROC-AUC style scores by integrating agreement across all possible thresholds.

\begin{align}
\bar{A}(l, l_p)
&=
\int_{0}^{1} A(l, l_p, \alpha)\, d\alpha
\nonumber \\
&=
\frac{1}{|\mathcal{C}|}
\sum_{c \in \mathcal{C}}
\int_{0}^{1}
\Bigl[
    \mathbb{I}(c, l, \alpha)\,
    \mathbb{I}(c, l_p, \alpha)
\nonumber \\
&\qquad\quad
    +
    \bigl(1 - \mathbb{I}(c, l, \alpha)\bigr)
    \bigl(1 - \mathbb{I}(c, l_p, \alpha)\bigr)
\Bigr]
\label{eq:integrated_conceptual_state_alignment}
\end{align}

$\bar{A}(l, l_p)$ provides a holistic view of the model's conceptual symmetry between any two layers of interest.
It is high for the cases where both layers correspond to all concepts under consideration in the same fashion, and is low otherwise. 
In~\autoref{appendix:tas_derivation}, we show that the $\bar{A}(l, l_p)$ can be expressed more conveniently in terms of the difference of TCAV scores between two layers. Specifically:

\begin{equation}
\begin{aligned}
    \bar{A}(l, l_p)
     &= \frac{1}{|\mathcal{C}|} \sum_{c \in \mathcal{C}} 1 - \left \lvert \mathrm{TCAV}_{c,l} - \mathrm{TCAV}_{c,l_p} \right \lvert
\end{aligned}
\label{eq:integrated_conceptual_state_alignment}
\end{equation}

\subsection{Derivation: TCAV Agreement Score}
\label{appendix:tas_derivation}

To derive the closed form expression for $\bar{A}(l, l_p)$ as defined in \autoref{eq:integrated_conceptual_state_alignment}, let $T_{c,l}, T_{c,l_p} \in [0, 1]$ denote $\mathrm{TCAV}_{c,l}$ and $\mathrm{TCAV}_{c,l_p}$ respectively. The integrated agreement for a single concept $c \in \mathcal{C}$ is:
\begin{equation}
\begin{aligned}
    \bar{A}_c &= \int_{0}^{1} \Bigl[ \mathbb{I}(T_{c,l} > \alpha)\mathbb{I}(T_{c,l_p} > \alpha) \\
    &\quad + (1 - \mathbb{I}(T_{c,l} > \alpha))(1 - \mathbb{I}(T_{c,l_p} > \alpha)) \Bigr] d\alpha
\end{aligned}
\end{equation}
The first term in the integrand is $1$ if $\alpha < \min(T_{c,l}, T_{c,l_p})$ and $0$ otherwise. The second term is $1$ if $\alpha \geq \max(T_{c,l}, T_{c,l_p})$ and $0$ otherwise. Evaluating the integrals separately yields:
\begin{equation}
\begin{aligned}
    \bar{A}_c &= \int_{0}^{\min(T_{c,l}, T_{c,l_p})} 1 \, d\alpha + \int_{\max(T_{c,l}, T_{c,l_p})}^{1} 1 \, d\alpha \\
    &= \min(T_{c,l}, T_{c,l_p}) + \left( 1 - \max(T_{c,l}, T_{c,l_p}) \right) \\
    &= 1 - \left( \max(T_{c,l}, T_{c,l_p}) - \min(T_{c,l}, T_{c,l_p}) \right) \\
    &= 1 - |T_{c,l} - T_{c,l_p}|
\end{aligned}
\end{equation}
Substituting this result back into the normalized sum over the concept library $\mathcal{C}$ provides the final closed-form expression:
\begin{equation}
    \bar{A}(l, l_p) = \frac{1}{|\mathcal{C}|} \sum_{c \in \mathcal{C}} \left( 1 - | \mathrm{TCAV}_{c,l} - \mathrm{TCAV}_{c,l_p} | \right)
\end{equation}
This confirms that the integrated agreement is exactly the complement of the Mean Absolute Error (MAE) between the TCAV scores of two layers.
\section{The E-TCAV Framework}
\label{sec:etcav-def}
The degeneracy of TCAV scores, discussed in~\autoref{sec:theorem-fast-proxy}, motivates the E-TCAV framework.
E-TCAV is designed to enable meaningful estimation of conceptual alignment in the final block of a neural network without requiring explicit probing of each layer.
To this end, we propose using the penultimate layer as the default probing layer when the objective is to analyze only the final block of the network.
As shown in~\autoref{sec:Alignment of TCAV Scores}, the TCAV layer agreement score, defined in~\autoref{metric:agreement}, remains predominantly above $0.75$ up to four layers preceding the penultimate layer, rendering explicit probing of these layers redundant.
In practice, the E-TCAV framework is applied as follows:
\begin{enumerate}
    \item Employ SignalCAV as the latent classifier.
    \item When the layer of interest lies within the final five layers of the model, use the TCAV scores computed at the penultimate layer, following the procedure described in~\autoref{theorem:fast-penul}.
\end{enumerate}

\subsubsection{Complexity Analysis}
\label{sec:complexity-analysis}
To quantify the efficiency gains of E-TCAV, we provide a comparative complexity analysis of its performance with the original TCAV method.
Let $N$ denote the number of test samples in class $k$ ($|X_k|$), and $M$ denote the number of samples used to train the latent classifier to obtain the concept activation vector.
Let $D_l$ and $D_p$ represent the dimensionality of the activations at layer $l$ and the penultimate layer, respectively.
Two phases dominate the computational cost of standard TCAV at an arbitrary layer $l$:
\begin{enumerate}
    \item \textbf{CAV Training:} Training a linear classifier (e.g., SVM or Logistic Regression) on $M$ samples in $D_l$-dimensional space. This has a complexity of approximately $O(\mathcal{T}(M, D_l))$, where $\mathcal{T}$ is the solver's complexity.
    \item \textbf{Sensitivity Calculation:} This requires $N$ forward passes and $N$ partial backward passes to compute the gradient $\nabla_{a_l} h_k$. 
    Let $\mathcal{C}_{fp}(l)$ be the cost of the forward pass till layer $l$ and $\mathcal{C}_{bp}(l)$ be the cost of a backward pass to layer $l$. The total cost is $O(N \cdot (\mathcal{C}_{fp}(l) + \mathcal{C}_{bp}(l) + D_l))$.
\end{enumerate}
Thus, the total time complexity of the standard TCAV procedure, denoted by $\mathcal{Q_{TCAV}}$, is following:
\[
\mathcal{Q}_{\mathrm{TCAV}}
= O\!\left(
\mathcal{T}(M, D_l)
+ N \bigl(
\mathcal{C}_{\mathrm{fp}}(l)
+ \mathcal{C}_{\mathrm{bp}}(l)
+ D_l
\bigr)
\right)
\]
The E-TCAV, on the other hand, does not need any forward or backward pass on the evaluation samples $|X_k|$. Using~\autoref{theorem:fast-penul} we can replace the directional derivative with the dot product, whose complexity is $O(D_l)$. Therefore, the total time complexity of the E-TCAV procedure, denoted by $\mathcal{Q_{E-TCAV}}$, is:
\[
\mathcal{Q}_{\mathrm{E-TCAV}}
= O\Bigl( \mathcal{T}(M, D_l) + D_l \Bigr)
\]
The efficiency gain obtained by using E-TCAV depends on the layer whose TCAV score is approximated using the penultimate layer.
When the target layer is the penultimate layer itself, the reduction in time complexity scales with the number of evaluation samples $N$ and the network size, as determined by the forward- and backward-pass costs $C_{fp}(l_p)$ and $C_{bp}(l_p)$.

The efficiency gain is substantially larger when approximating TCAV scores for an earlier layer $l$.
In particular, when $D_l \ll D_p$, the difference in solver complexity $\mathcal{T}$ further contributes to the overall gain.

\section{Experiment and Results}
In this section, we provide evidence for the existence of the inter-layer TCAV agreements in the final block of the neural network, when SignalCAV, as described in the E-TCAV framework, is employed.
Our experiments span five benchmark datasets. Below, we first explain the key characteristics of these datasets, followed by the empirical analyses on fidelity, efficiency, and generalization of the proposed E-TCAV.
\subsection{Datasets}
\subsubsection{CelebA - Concepts With Ground Truth} 
\label{ds:celeba}
The first dataset used in this study is CelebA - Large-scale CelebFaces Attributes dataset~\cite{liu2015faceattributes}, a benchmark for facial analysis. It contains 40 binary facial attributes per image, such as Male, Young, Eyeglasses, Beard, Bald, etc.
The natural distribution of the dataset is such that it contains known confounders. For example, about $99.35\%$ of all people in the dataset with blonde hair don't wear a necktie. 
These correlations can be used to approximate ground truth results of the TCAV testing.
For example, a network trained on the hair color prediction task may learn the necktie feature as a shortcut, and therefore,e we can expect low TCAV scores of the necktie concept for the blonde class.

This information is particularly helpful for analyzing if the latent classifier has picked up the right concept direction or not, and is more reliable since the facial attributes are more complex and are better representative of common use case scenarios than general synthetic datasets.
\textit{See Appendix~\autoref{apx:celeba} for more details.}
\subsubsection{SCDB - Simple Concept Database}
SCDB is a large-scale synthetic dataset developed originally for concept localization~\cite{lucieri2020explaining} - inspired by the challenges of the skin lesion classification. The dataset images contain big geometric shapes containing multiple concepts. Each concept is represented as a smaller geometry inside the bigger one with random color, shape, and orientation.
The dataset images are divided into two classes - each class is a combination of multiple concepts. There are total of ten available concepts. 
We use SCDB because of the diverse nature of its concepts and the availability of the concept segmentation masks that allow precise concept definition. \textit{See Appendix~\autoref{apx:scdb} for more details.}
\subsubsection{ISIC - Skin Lesions}
We extend the scope of our study by using the dermatoscopic images from the International Skin Imaging Collaboration dataset~\cite{codella2017isbi,tschandl2018ham10000,hernandezperez2024bcn20000} (ISIC-2019) to train different networks for the binary classification of melanoma.
We then use the derm7pt dataset~\cite{kawahara2018seven}, which contains diverse concepts related to melanoma disease as probe datasets for TCAV, to find how different concepts affect the network's decision.
\textit{See Appendix~\autoref{apx:isic} for more details.}
\subsubsection{ImageNet-1K}
Since the complexity of the dataset and the number of classes dictate the geometry of the decision boundaries, it directly determines the nature of the gradients across layers.
To account for use cases beyond binary classification, we evaluate the behavior of TCAV scores on the ImageNet-1K dataset~\cite{deng2009imagenet}.
We assess the statistical significance of the \textbf{Stripes} concept, derived from the Broden dataset~\cite{bau2017network}, across multiple animal classes, including zebra, lion, leopard, and tiger.
\subsubsection{Jigsaw / Wikipedia Toxicity Dataset}  
To ensure the generalization of our findings beyond computer vision, we include the Wikipedia Toxicity dataset~\cite{davidson2017automated}, a comment-level text classification dataset built to study and detect toxic behavior in online discussions. 
In particular, we analyze the significance of \textbf{insult}, \textbf{obscene}, and \textbf{neutral} concepts. \textit{See Appendix~\autoref{apx:wiki-toxic} for more details.}

\subsection{Evaluations}
\begin{table}[t!]
\centering
\small
\begin{tabularx}{\linewidth}{Xccc}
\toprule
Layer & Signal & SVM & $\Delta$ \\
\midrule
\multicolumn{4}{l}{\textbf{resnet50}} \\
\midrule
layer3.5 & 1.000 $\pm$ 0.000 & 1.000 $\pm$ 0.000 & +0.000 \\
layer4.0 & 1.000 $\pm$ 0.000 & 0.960 $\pm$ 0.136 & \textcolor{green!40!black}{+0.040} \\
layer4.1 & 1.000 $\pm$ 0.000 & 1.000 $\pm$ 0.000 & \textcolor{green!40!black}{+0.000} \\
layer4.2 & 1.000 $\pm$ 0.000 & 0.967 $\pm$ 0.183 & \textcolor{green!40!black}{+0.033} \\
avgpool & 1.000 $\pm$ 0.000 & 1.000 $\pm$ 0.000 & +0.000 \\
\midrule
\multicolumn{4}{l}{\textbf{densenet121}} \\
\midrule
denselayer13 & 1.000 $\pm$ 0.000 & 0.941 $\pm$ 0.146 & \textcolor{green!40!black}{+0.059} \\
denselayer14 & 1.000 $\pm$ 0.000 & 0.988 $\pm$ 0.053 & \textcolor{green!40!black}{+0.012} \\
denselayer15 & 1.000 $\pm$ 0.000 & 0.521 $\pm$ 0.401\rlap{\textcolor{red}{*}} & \textcolor{green!40!black}{+0.479} \\
denselayer16 & 1.000 $\pm$ 0.000 & 0.452 $\pm$ 0.340\rlap{\textcolor{red}{*}} & \textcolor{green!40!black}{+0.548} \\
avgpool & 1.000 $\pm$ 0.000 & 1.000 $\pm$ 0.000 & +0.000 \\
\midrule
\multicolumn{4}{l}{\textbf{inception\_v3}} \\
\midrule
Mixed\_6e & 1.000 $\pm$ 0.000 & 0.871 $\pm$ 0.172 & \textcolor{green!40!black}{+0.129} \\
Mixed\_7a & 1.000 $\pm$ 0.000 & 0.917 $\pm$ 0.165 & \textcolor{green!40!black}{+0.083} \\
Mixed\_7b & 1.000 $\pm$ 0.000 & 0.875 $\pm$ 0.186 & \textcolor{green!40!black}{+0.125} \\
Mixed\_7c & 1.000 $\pm$ 0.000 & 0.533 $\pm$ 0.507\rlap{\textcolor{red}{*}}& \textcolor{green!40!black}{+0.467} \\
avgpool & 1.000 $\pm$ 0.000 & 1.000 $\pm$ 0.000 & +0.000 \\
\bottomrule
\end{tabularx}
\caption{TCAV scores for the Necktie concept evaluated on the class of non-blonde hair. Values are shown as mean $\pm$ std. \textcolor{red}{*} indicates $p > 0.05$. All mentioned scores have accuracy $> 75\%$.}\label{tab:celeba-scores-non-blond}
\end{table}

\begin{table}[t!]
\centering
\small
\begin{tabularx}{\linewidth}{>{\raggedright\arraybackslash}p{2cm} 
                            >{\centering\arraybackslash}X
                            >{\centering\arraybackslash}X
                            >{\centering\arraybackslash}X}
\toprule
Layer & \multicolumn{3}{c}{Wiki} \\
\cmidrule(lr){2-4}
  & Signal & SVM & $\Delta$ \\
\midrule
\multicolumn{3}{l}{\textbf{roberta-base}} \\
\midrule
layer.11 & \textbf{1.000} & 0.558 & \textcolor{green!40!black}{+0.442} \\
layer.10 & \textbf{1.000} & 0.697 & \textcolor{green!40!black}{+0.303} \\
layer.9 & \textbf{1.000} & 0.440 & \textcolor{green!40!black}{+0.560} \\
layer.8 & \textbf{1.000} & 0.490 & \textcolor{green!40!black}{+0.510} \\
layer.7 & \textbf{1.000} & 0.474 & \textcolor{green!40!black}{+0.526} \\
\midrule
\end{tabularx}
\caption{Layer-wise agreement scores for Signal and SVM classifiers. Signal and SVM results are shown side-by-side for each dataset to enable direct comparison. Values greater than 0.75 are highlighted in bold.}
\label{tab:layer-agreement-nlp}
\end{table}
\begin{table*}[t!]
\centering
\small
\begin{tabularx}{\textwidth}{l *{12}{>{\centering\arraybackslash}X}}
\toprule
Layer & \multicolumn{3}{c}{CelebA} & \multicolumn{3}{c}{ISIC} & \multicolumn{3}{c}{SCDB} & \multicolumn{3}{c}{ImageNet} \\
\cmidrule(lr){2-4} \cmidrule(lr){5-7} \cmidrule(lr){8-10} \cmidrule(lr){11-13}
  & Signal & SVM & $\Delta$ & Signal & SVM & $\Delta$ & Signal & SVM & $\Delta$ & Signal & SVM & $\Delta$ \\
\midrule
\multicolumn{13}{l}{\textbf{inception\_v3}} \\
\midrule
Mixed\_7c & \textbf{1.000} & 0.420 & \textcolor{green!40!black}{+0.580} & \textbf{1.000} & 0.708 & \textcolor{green!40!black}{+0.292} & \textbf{0.963} & 0.742 & \textcolor{green!40!black}{+0.222} & \textbf{1.000} & \textbf{0.942} & \textcolor{green!40!black}{+0.058} \\
Mixed\_7b & \textbf{1.000} & \textbf{0.876} & \textcolor{green!40!black}{+0.124} & \textbf{0.976} & 0.689 & \textcolor{green!40!black}{+0.286} & \textbf{0.843} & 0.605 & \textcolor{green!40!black}{+0.238} & \textbf{0.984} & \textbf{0.773} & \textcolor{green!40!black}{+0.211} \\
Mixed\_7a & \textbf{1.000} & \textbf{0.960} & \textcolor{green!40!black}{+0.040} & \textbf{0.963} & \textbf{0.764} & \textcolor{green!40!black}{+0.199} & \textbf{0.784} & \textbf{0.792} & \textcolor{red!40!black}{-0.008} & \textbf{0.787} & \textbf{0.913} & \textcolor{red!40!black}{-0.125} \\
Mixed\_6e & \textbf{1.000} & \textbf{0.946} & \textcolor{green!40!black}{+0.054} & \textbf{0.954} & \textbf{0.789} & \textcolor{green!40!black}{+0.166} & 0.237 & 0.467 & \textcolor{red!40!black}{-0.230} & 0.438 & 0.722 & \textcolor{red!40!black}{-0.283} \\
Mixed\_6d & \textbf{1.000} & \textbf{0.947} & \textcolor{green!40!black}{+0.053} & 0.720 & 0.749 & \textcolor{red!40!black}{-0.029} & 0.257 & 0.306 & \textcolor{red!40!black}{-0.049} & 0.665 & \textbf{0.810} & \textcolor{red!40!black}{-0.145} \\
\midrule
\multicolumn{13}{l}{\textbf{resnet50}} \\
\midrule
layer4.2 & \textbf{1.000} & \textbf{0.913} & \textcolor{green!40!black}{+0.087} & \textbf{1.000} & \textbf{0.762} & \textcolor{green!40!black}{+0.238} & \textbf{0.955} & \textbf{0.833} & \textcolor{green!40!black}{+0.122} & \textbf{1.000} & \textbf{0.950} & \textcolor{green!40!black}{+0.050} \\
layer4.1 & \textbf{1.000} & \textbf{0.944} & \textcolor{green!40!black}{+0.056} & \textbf{0.998} & 0.735 & \textcolor{green!40!black}{+0.263} & \textbf{0.906} & 0.646 & \textcolor{green!40!black}{+0.260} & 0.662 & \textbf{0.932} & \textcolor{red!40!black}{-0.269} \\
layer4.0 & \textbf{1.000} & \textbf{0.962} & \textcolor{green!40!black}{+0.038} & \textbf{0.989} & \textbf{0.756} & \textcolor{green!40!black}{+0.233} & \textbf{0.804} & 0.619 & \textcolor{green!40!black}{+0.186} & 0.631 & \textbf{0.795} & \textcolor{red!40!black}{-0.164} \\
layer3.5 & \textbf{1.000} & \textbf{1.000} & +0.000 & \textbf{0.903} & 0.678 & \textcolor{green!40!black}{+0.224} & \textbf{0.783} & 0.567 & \textcolor{green!40!black}{+0.216} & 0.352 & 0.540 & \textcolor{red!40!black}{-0.188} \\
layer3.4 & \textbf{1.000} & \textbf{1.000} & +0.000 & \textbf{0.885} & 0.708 & \textcolor{green!40!black}{+0.176} & 0.650 & 0.595 & \textcolor{green!40!black}{+0.055} & 0.244 & 0.443 & \textcolor{red!40!black}{-0.199} \\
\midrule
\multicolumn{13}{l}{\textbf{densenet121}} \\
\midrule
denselayer16 & 0.200 & 0.515 & \textcolor{red!40!black}{-0.315} & \textbf{0.896} & \textbf{0.870} & \textcolor{green!40!black}{+0.025} & \textbf{0.935} & 0.627 & \textcolor{green!40!black}{+0.308} & \textbf{0.975} & \textbf{0.898} & \textcolor{green!40!black}{+0.078} \\
denselayer15 & \textbf{0.999} & 0.711 & \textcolor{green!40!black}{+0.288} & 0.548 & \textbf{0.831} & \textcolor{red!40!black}{-0.283} & \textbf{0.929} & 0.700 & \textcolor{green!40!black}{+0.229} & 0.744 & \textbf{0.775} & \textcolor{red!40!black}{-0.031} \\
denselayer14 & \textbf{0.964} & \textbf{0.844} & \textcolor{green!40!black}{+0.119} & \textbf{0.793} & \textbf{0.818} & \textcolor{red!40!black}{-0.026} & \textbf{0.962} & 0.694 & \textcolor{green!40!black}{+0.268} & \textbf{0.950} & \textbf{0.826} & \textcolor{green!40!black}{+0.125} \\
denselayer13 & 0.529 & \textbf{0.755} & \textcolor{red!40!black}{-0.226} & \textbf{0.888} & \textbf{0.846} & \textcolor{green!40!black}{+0.042} & \textbf{0.923} & 0.602 & \textcolor{green!40!black}{+0.320} & 0.705 & \textbf{0.781} & \textcolor{red!40!black}{-0.076} \\
denselayer11 & \textbf{1.000} & \textbf{0.788} & +\textcolor{green!40!black}{0.212} & \textbf{0.947} & \textbf{0.848} & \textcolor{green!40!black}{+0.099} & \textbf{0.877} & 0.628 & \textcolor{green!40!black}{+0.250} & \textbf{0.992} & \textbf{0.890} & \textcolor{green!40!black}{+0.103} \\
\midrule
\end{tabularx}
\caption{Layer-wise agreement scores for Signal and SVM classifiers. Signal and SVM results are shown side-by-side for each dataset to enable direct comparison. Values greater than 0.75 are highlighted in bold.}
\label{tab:layer-agreement-vision}
\end{table*}
\subsubsection{Validating the Concept Direction}
\label{sec:Validating the Concept Direction}
The results for TCAV testing on CelebA~\autoref{ds:celeba} are summarized in~\autoref{tab:celeba-scores-non-blond}.
It can be observed that the TCAV scores for the \textbf{necktie} concept for the \textbf{non-blonde} class are near one. It is worth noting that instead of using a necktie and random samples, we use samples with and without a necktie for the concept extraction.
Therefore, based on the high TCAV scores for the necktie, we can infer that the model will decrease the score for the non-blond class if there is an absence of the necktie concept in the input samples.
Both the SVM and the SignalCAV\footnote{The original paper~\cite{pahde2022navigating} details the efficacy of the SignalCAV.}
 have identified the right concept direction.
Therefore, it is safe to conclude that SignalCAV indeed captures the right concept direction, and its results can be trusted equally.

\subsubsection{Inter-Layer TCAV Agreement}
\label{sec:Alignment of TCAV Scores}
To evaluate the alignment of the TCAV scores, we compute the layer-wise TCAV agreement scores,~\autoref{metric:agreement}, between the TCAV scores from the penultimate layer of each network and the preceding layers. 
The TCAV agreement scores for the last five layers of InceptionV3, DenseNet121, and Resnet50 are presented in~\autoref{tab:layer-agreement-vision}. 
The results are computed by applying TCAV testing on CelebA, Simple Concept Dataset, ISIC-2019, and ImageNet datasets.
The results from the Wikipedia Toxicity dataset are presented in~\autoref{tab:layer-agreement-nlp}.

Our concept library consists of $28$ unique concepts. These results are obtained from a total of $50$ different class-concept combinations. The details of the concept curation per dataset, including the choice of positive, negative, and evaluation samples is described in the~\autoref{apx:concepts_ds}.
\textit{The raw TCAV scores and the code base for these experiments are given in the supplementary materials. }

The results indicate that TCAV scores from the layers preceding the penultimate layer generally agree with those at the penultimate layer. For the case of vision models~\autoref{tab:layer-agreement-vision}, we choose the \texttt{average pooling} as the penultimate layer. 
However, in the case of a text-classification model, Roberta-Base, the default classification head is not affine and contains nonlinear activations. 
Therefore, to fulfill the assumption of~\autoref{theorem:fast-penul}, we use the \texttt{dropout layer} of the classification head as the penultimate layer, since there is no nonlinearity after this layer.
Using this method of layer selection, E-TCAV can be applied to any model whose default classification head is nonlinear.

Though the exact degree of agreement varies with the dataset and network architecture, it is predominantly greater than $0.75$ till $4^{th}$ layer before the penultimate one, indicating a strong structural alignment in the final block of neural networks.

Secondly, we observe that the use of SignalCAV enhances the inter-layer agreement. For most of the cases, the agreement delta, presented in the third columns of~\autoref{tab:layer-agreement-nlp} and~\autoref{tab:layer-agreement-vision}, is positive, indicating that the agreement score for SignalCAV is more than that of SVM.

However, this trend does not extend to the case of ImageNet, where using SVM results in increased alignment of TCAV scores.

We hypothesize that this is likely because the semantic alignment is dependent on the complexity of the problem/dataset, which dictates the geometry of the resulting decision boundaries.
And since the SignalCAV effectively reduces to the difference of means~\autoref{sec:Signal's Statistical Superiority}, it might be the case that the concept activation vector obtained in this case is inferior to the one obtained by the SVM classifier.

However, it is worth noting that the E-TCAV framework is, in principle, agnostic to the choice of latent classifier. Since the SVM classifier yields high TCAV agreement scores on ImageNet, E-TCAV remains applicable in complex settings; in such cases, we recommend using SVM instead of SignalCAV.

\begin{figure}[t!]
    \centering
    \includegraphics[width=1\linewidth]{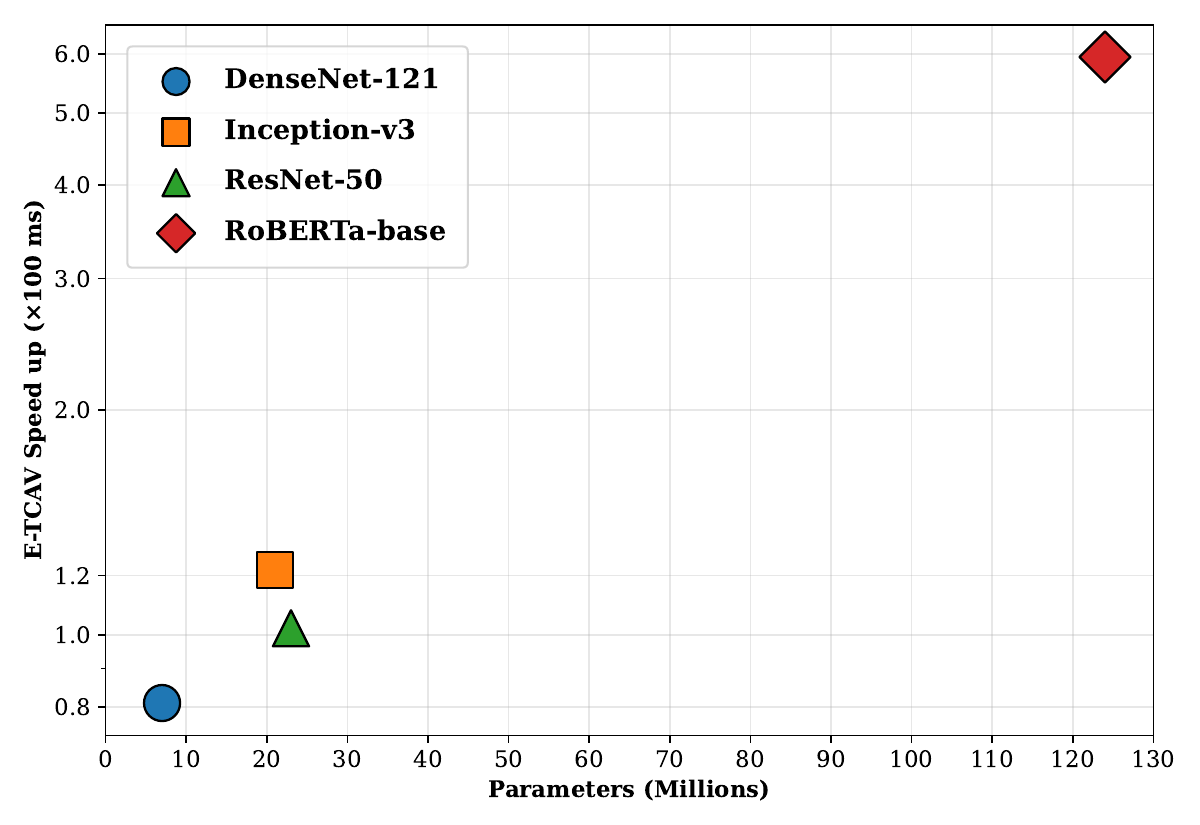}
    \caption{Linearly scaling E-TCAV speedup. The x-axis is the number of parameters; the y-axis represents the runtime reduction achieved by the proposed E-TCAV framework.}
    \label{fig:alignment_summary}
\end{figure}
\subsubsection{Efficiency of E-TCAV}
To numerically quantify the benefits of using the penultimate layer as the proxy, we perform the standard TCAV testing using various pretrained models, across multiple layers, on the CelebA dataset~\autoref{ds:celeba}. We then perform the TCAV testing on the penultimate layer of each network for the same concepts while using the definition provided in~\autoref{sec:theorem-fast-proxy}. We use the same latent classifiers and the hardware for both standard and proposed implementations.

The~\autoref{tab:layer-speedup} shows the percentage increase in the computational efficiency relative to the standard implementation using the proposed method, calculated as $( TIME_{TCAV} - TIME_{E-TCAV} )/ TIME_{TCAV}$.
Though the gain depends on the network architecture and the number of evaluation samples, it tends to increase as we approximate the TCAV scores for deeper layers, since the size of the latent space increases with an increase in depth, which needs more calculations for both the extraction of concept activation vector and the calculation of directional derivatives.

The wall clock comparison between E-TCAV and traditional TCAV is shown in Fig. TODO. The results are obtained by analyzing the runtime of both techniques when applied to the penultimate layers of the respective models. The a-axis is the number of parameters of the model, and the y-axis is the runtime difference ($TIME_{TCAV} -TIME_{E-TCAV} $).
As projected by the complexity analysis given in~\autoref{sec:complexity-analysis}, the gain can be seen growing linearly with the number of parameters.

\begin{table}[t!]
\centering
\small
\begin{tabularx}{\linewidth}{Xc}
\toprule
Layer & Relative Speedup \\
\midrule
\multicolumn{2}{l}{\textbf{resnet50}} \\
\midrule
layer4.0 & 31.85\% \\
layer4.1 & 30.75\% \\
layer4.2 & 12.53\% \\
avgpool & 23.43\% \\
\midrule
\multicolumn{2}{l}{\textbf{densenet121}} \\
\midrule
denselayer14 & 20.02\% \\
denselayer15 & 22.84\% \\
denselayer16 & 23.48\% \\
avgpool & 17.99\% \\
\midrule
\multicolumn{2}{l}{\textbf{inception\_v3}} \\
\midrule
Mixed\_7a & 28.45\% \\
Mixed\_7b & 29.30\% \\
Mixed\_7c & 28.92\% \\
avgpool & 16.39\% \\
\bottomrule
\end{tabularx}
\caption{Relative execution-time speedup of E-TCAV compared to default TCAV across increasing network depth.}
\label{tab:layer-speedup}
\end{table}

\subsubsection{Signal's Statistical Superiority}
\label{sec:Signal's Statistical Superiority}
To demonstrate the benefits of using SignalCAV we conduct further analysis by comparing different statistics derived from the TCAV testing using each of the classifiers. The results are presented in~\autoref{tab:master-std-pval}.
It is clear to observe that the average of the standard deviation of the SignalCAV is quite below that of the SVM.
Similarly, the average p-values obtained after double-sided t-testing are lower for the SignalCAV as compared to the SVM.

The stability of SignalCAV can be attributed to the relatively deterministic process of its calculation. 
The calculation of the concept activation vector is dependent on two factors: the non-concept samples used and the convergence behavior of the latent classifier.

The SVM objective is defined by the samples closest to the decision boundary and hence is very sensitive to the specific geometry of the randomly sampled non-concept set.
However, the SignalCAV reduces to the difference of the means of the activations in case of binary labels~\cite{pahde2022navigating}.
\[\boldsymbol{v}_C \propto \mathbb{E}[\boldsymbol{a}_{concept}] - \mathbb{E}[\boldsymbol{a}_{random}]\]
As a global statistical measure, the mean is a more stable estimator when compared to the support margins, and hence it provides a smoother representation resulting in low variance.
\begin{table}[t!]
\centering
\small
\setlength{\tabcolsep}{5pt}
\begin{tabularx}{\linewidth}{Xcccccc} 
\toprule
Dataset & \multicolumn{3}{c}{std} & \multicolumn{3}{c}{p-val} \\
\cmidrule(lr){2-4} \cmidrule(lr){5-7}
 & Signal & SVM & $\Delta$ & Signal & SVM & $\Delta$ \\
\midrule
\multicolumn{7}{l}{\textbf{resnet50}} \\
\midrule
CelebA & 0.000 & 0.042 & \textcolor{green!40!black}{-0.042} & 0.000 & 0.000 & -0.000 \\
ISIC & 0.111 & 0.268 & \textcolor{green!40!black}{-0.157} & 0.029 & 0.056 & \textcolor{green!40!black}{-0.028} \\
ImageNet & 0.047 & 0.134 & \textcolor{green!40!black}{-0.087} & 0.012 & 0.052 & \textcolor{green!40!black}{-0.041} \\
SCDB & 0.123 & 0.120 & \textcolor{red!40!black}{+0.004} & 0.051 & 0.031 & \textcolor{red!40!black}{+0.020} \\
\midrule
\multicolumn{7}{l}{\textbf{densenet121}} \\
\midrule
CelebA & 0.003 & 0.149 & \textcolor{green!40!black}{-0.146} & 0.000 & 0.032 & \textcolor{green!40!black}{-0.032} \\
ISIC & 0.229 & 0.314 & \textcolor{green!40!black}{-0.085} & 0.081 & 0.170 & \textcolor{green!40!black}{-0.088} \\
ImageNet & 0.079 & 0.127 & \textcolor{green!40!black}{-0.047} & 0.006 & 0.015 & \textcolor{green!40!black}{-0.010} \\
SCDB & 0.070 & 0.132 & \textcolor{green!40!black}{-0.062} & 0.009 & 0.034 & \textcolor{green!40!black}{-0.024} \\
\midrule
\multicolumn{7}{l}{\textbf{inception\_v3}} \\
\midrule
CelebA & 0.000 & 0.082 & \textcolor{green!40!black}{-0.082} & 0.000 & 0.013 & \textcolor{green!40!black}{-0.013} \\
ISIC & 0.123 & 0.251 & \textcolor{green!40!black}{-0.127} & 0.025 & 0.115 & \textcolor{green!40!black}{-0.091} \\
ImageNet & 0.040 & 0.120 & \textcolor{green!40!black}{-0.079} & 0.000 & 0.055 & \textcolor{green!40!black}{-0.055} \\
SCDB & 0.106 & 0.095 & \textcolor{red!40!black}{+0.011} & 0.026 & 0.022 & \textcolor{red!40!black}{+0.003} \\
\midrule
\multicolumn{7}{l}{\textbf{roberta-base}} \\
\midrule
wiki & 0.000 & 0.162 & \textcolor{green!40!black}{-0.162} & 0.000 & 0.032 & \textcolor{green!40!black}{-0.032} \\
\bottomrule
\end{tabularx}
\caption{Comparison of standard deviations and p-values of concept scores across models and datasets using Signal and SVM classifiers. In this specific table, lower value is better, and as we can observe Signal has a consistently lower standard deviation and p-value across the majority of test cases.}
\label{tab:master-std-pval}
\end{table}

\section{Conclusion}
In this study, we addressed the variability and computational cost of TCAV scores. 
We showed that both inter-layer variability and differences across multiple TCAV runs are largely influenced by the choice of latent classifier, with SignalCAV producing higher inter-layer agreement and lower variance. 
Analytically, we demonstrated that TCAV scores are degenerate at the penultimate layer for affine classifiers. 
Leveraging this, we proposed the E-TCAV framework, which uses the penultimate layer as a proxy for preceding layers to compute TCAV scores, achieving computational speed-ups that scale linearly with the number of evaluation samples and the network size.

The proposed E-TCAV framework provides several practical benefits. First, it reduces the computational cost of standard TCAV, making it useful not only for post-hoc interpretability but also for concept-guided training, where efficient per-batch computation can compound into substantial savings. Second, it offers a principled way to guide layer selection in general settings: our results indicate that the final block consistently preserves meaningful semantic information, making penultimate a reliable target for probing.

Experiments on multiple datasets, CelebA, ISIC, SCDB, ImageNet, and the Wikipedia Toxicity, demonstrate that the conceptual presence or absence is consistent across multiple layers of networks. The diversity of the datasets implies that results apply to both computer vision and natural language domains.

We further observe that the TCAV scores increasingly diverge from those of the penultimate layer in deeper layers of the networks. Consequently, E-TCAV is most reliable when applied to the final block of the network, where semantic representations are expected to be consistent. Future work should investigate approximation strategies beyond direct layer substitution, with the goal of improving robustness while preserving the computational advantages of E-TCAV.
\clearpage
\bibliographystyle{named}
\bibliography{ref}
\appendix

\clearpage

\section{Datasets \& Concepts Curation}
\label{apx:concepts_ds}
\subsection{CelebA}
\label{apx:celeba}
The original CelebA dataset is first divided into train, validation, and test splits. 
A train split is used to train the model on the base task.
Each concept is represented by $200$ images extracted from the validation split.
While evaluation samples used in the calculation of directional sensitivities are extracted from the test split.
The CelebA dataset has been used in the following two different configurations in this work, each involving a different based task.

\subsubsection{Hair Color Classification}
In ~\autoref{sec:Validating the Concept Direction}, we used this configuration to check if the classifiers capture the true concept direction.
We constructed a dataset to detect if the hair color of the subject is blond or not. We used the \textbf{Blond\_Hair} attribute to identify the relevant classes.
We then performed the TCAV testing on the \textbf{Necktie} concept for both of the classes. The probe dataset was constructed as follows: 

\begin{itemize}
    \item \textbf{Positive Samples}: Randomly sampling with replacement from the validation split, where \texttt{Wearing\_Necktie} == 1.
    \item \textbf{Negative Samples}: Randomly sampling with replacement from the validation split, where \texttt{Wearing\_Necktie} == -1.    
\end{itemize}

The evaluation samples for each class were drawn from the test set, similar to how it was done in the training set.

\subsubsection{Gender Classification}
For evaluating the inter-layer agreements of the TCAV scores in ~\autoref{sec:Alignment of TCAV Scores}, in addition to the above configuration, we included another configuration of the CelebA with more concepts to have generalizable results.
We constructed a dataset for the gender classification of the subject. We used the \textbf{Male} attribute to identify the relevant classes.
We then performed the TCAV testing on the concepts of \textbf{Bald}, \textbf{Mustache}, \textbf{Wearing\_Lipstick}, and \textbf{Wearing\_Necktie}.
Similar to the previous configuration, concepts were extracted from the validation split, where positive samples were identified by the \texttt{Attribute} == 1 and negative samples by \texttt{Attribute} == 1.

\subsection{ISIC}
\label{apx:isic}
For all of the experiments where the ISIC dataset~\cite{codella2017isbi,tschandl2018ham10000,hernandezperez2024bcn20000} was used, we trained the base models on the binary classification of \textbf{Melanoma (MEL)}.
The dataset was constructed in one versus all fashion (MEL vs 7 other classes).
The concepts were constructed using the \textbf{Derm7pt} dataset~\cite{kawahara2018seven}, which contains the sample-level annotations describing the nature of the image.
We construct eight concepts from the dermoscopic images provided in the dataset, and our technique is elaborated in~\autoref{table:isic-concept-mapping}. Each concept is represented by $70$ samples in this case.




\begin{table}[t]
\centering
\small
\setlength{\tabcolsep}{4pt}
\renewcommand{\arraystretch}{1.15}

\begin{tabularx}{\columnwidth}{
  >{\raggedright\arraybackslash}p{0.28\columnwidth}
  >{\raggedright\arraybackslash}p{0.28\columnwidth}
  >{\raggedright\arraybackslash}X
}
\toprule
\textbf{Concept Name} & \textbf{Derm7pt} & \textbf{Values} \\
\midrule

Regular Pigment Network
& \texttt{pigment network}
& \{\texttt{typical}\} \\

Irregular Pigment Network
& \texttt{pigment network}
& \{\texttt{atypical}\} \\

Regular Streaks
& \texttt{streaks}
& \{\texttt{regular}\} \\

Irregular Streaks
& \texttt{streaks}
& \{\texttt{irregular}\} \\

Regular Dots and Globules
& \texttt{dots and globules}
& \{\texttt{regular}\} \\

Irregular Dots and Globules
& \texttt{dots and globules}
& \{\texttt{irregular}\} \\

Typical Pigmentation
& \texttt{pigmentation}
& \{\texttt{diffuse regular}, \texttt{localized regular}\} \\

Atypical Pigmentation
& \texttt{pigmentation}
& \{\texttt{diffuse irregular}, \texttt{localized irregular}\} \\

Random Set
& \texttt{*}
& \{\texttt{absent}\} \\

\bottomrule
\end{tabularx}

\caption{Mapping between semantic concepts and annotation attributes in the Derm7pt dataset used to construct positive concept sets. Each concept is defined by a specific attribute in the original annotations and an associated set of categorical values treated as positive labels. A universal negative set is constructed by pooling samples in which the corresponding attribute is annotated as \texttt{absent}.}
\label{table:isic-concept-mapping}
\end{table}
\subsection{SCDB}
\label{apx:scdb}
The SCDB provides the concept images and the masks of the concepts present in each of the image which we use to construct the concept library. For experimentation, we use all available concepts, namely Cross, Ellipse, Hexagon, Line, Pentagon, Rectangle, Star, Starmarker, Triangle, and Tripod.
To perform TCAV testing for a given concept, we randomly choose $200$ samples of that concept, obtained after applying the provided concept masks on the given images. We construct 30 random sets per concept by uniformly sampling $200$ images for each of the sets from the remaining concepts.

\subsection{Wikipedia Toxicity}
\label{apx:wiki-toxic}
For the Wikipedia Toxicity dataset~\cite{davidson2017automated}, we first divided that dataset into train, validation, and test splits.
The model was trained on the binary classification of the text samples into \texttt{Toxic} or \texttt{Non-Toxic} class.
The dataset also assigns a binary value for various attributes of a given text sample. We use the \texttt{Obscene} and the \texttt{Insult} attribute to curate two different concepts from the test split of the dataset.
We also create a synthetic \texttt{neutral} concept by using Chat-GPT 4.1-mini, so as to create a non-toxic concept as well.
The negative samples were sampled from a random tweet dataset. And the evaluation samples for each class were drawn from the validation split.

\end{document}